\documentclass{article}

\input{1-preamble}
\usepackage[accepted]{icml2024}

\usepackage{titlesec}
\titlespacing*{\paragraph}{0pt}{0pt plus 1ex}{1em}
\titlespacing*{\subsubsection}{0pt}{0pt plus 1ex}{0pt}
\titlespacing*{\subsection}{0pt}{0pt plus 1ex}{0pt}
\titlespacing*{\section}{0pt}{2pt plus 1ex}{0pt}

\usepackage{todonotes}

\title{
    Learning Temporal Distances: Contrastive Successor Features Can Provide a Metric Structure for Decision-Making
}

\begin{document}

\makeatletter
\twocolumn[

    \icmltitle{\@title}

    \begin{icmlauthorlist}
        \icmlauthor{Vivek Myers}{berkeley}
        \icmlauthor{Chongyi Zheng}{princeton}
        \icmlauthor{Anca Dragan}{berkeley}
        \icmlauthor{Sergey Levine}{berkeley}
        \icmlauthor{Benjamin Eysenbach}{princeton}
    \end{icmlauthorlist}

    \icmlaffiliation{berkeley}{University of California, Berkeley}
    \icmlaffiliation{princeton}{Princeton University}

    \icmlcorrespondingauthor{Vivek Myers}{vmyers@berkeley.edu}

        \icmlkeywords{Contrastive Learning, Triangle Inequality, Quasimetric, Reinforcement Learning, Representation Learning}

    \vskip 0.3in
]
\makeatother

\printAffiliationsAndNotice{} 
\begin{abstract}
    Temporal distances lie at the heart of many algorithms for planning, control, and reinforcement learning that involve reaching goals, allowing one to estimate the transit time between two states.
However, prior attempts to define such temporal distances in stochastic settings have been stymied by an important limitation: these prior approaches do not satisfy the triangle inequality.
This is not merely a definitional concern, but translates to an inability to generalize and find shortest paths.
In this paper, we build on prior work in contrastive learning and quasimetrics to show how successor features learned by contrastive learning (after a change of variables) form a temporal distance that does satisfy the triangle inequality, even in stochastic settings.
Importantly, this temporal distance is computationally efficient to estimate, even in high-dimensional and stochastic settings.
Experiments in controlled settings and benchmark suites demonstrate that an RL algorithm based on these new temporal distances exhibits combinatorial generalization (i.e., ``stitching'') and can sometimes learn more quickly than prior methods, including those based on quasimetrics.

\end{abstract}

\definecolor{myOrange}{RGB}{198, 117, 61}
\definecolor{myGreen}{RGB}{    0, 123, 125}

\begin{figure*}[t]
    \centering
    \vspace*{-1ex}
    \makebox[\linewidth][c]{
        \captionsetup[sub]{labelfont=bf}
        \begin{subfigure}[t]{0.45\linewidth}
            \centering
            \includestandalone[width=\linewidth]{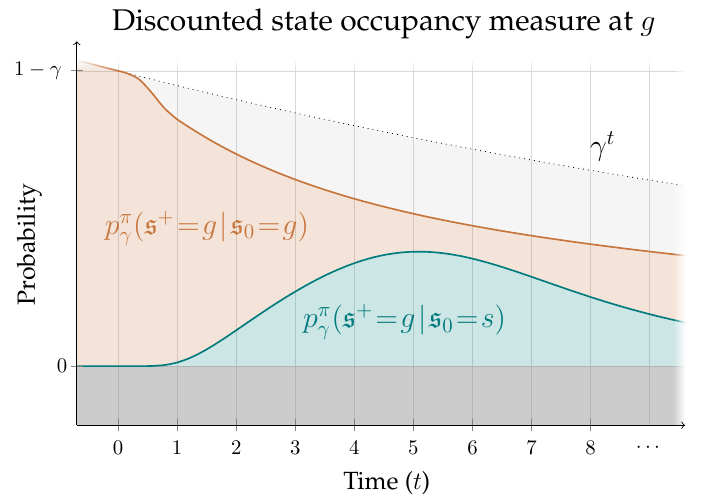}
            \subcaption{
            Starting at state $s$, we visualize the (discouted) probability of reaching state $g$ after exactly $t$ steps ({\color{myGreen}\gscolorname}). The sum of these probabilities ({\color{myGreen}$\blacksquare$ area}) is the probability of reaching state $g$ at some point in the future. Our method defines the distance between states $s$ and $g$ as the difference in these shaded areas ({\color{myGreen}$\blacksquare$ area} - {\color{myOrange}$\blacksquare$ area}).             }
        \end{subfigure}
                
        \kern 2em
        \begin{subfigure}[t]{0.5\linewidth}
            \centering
            \includestandalone[width=\linewidth]{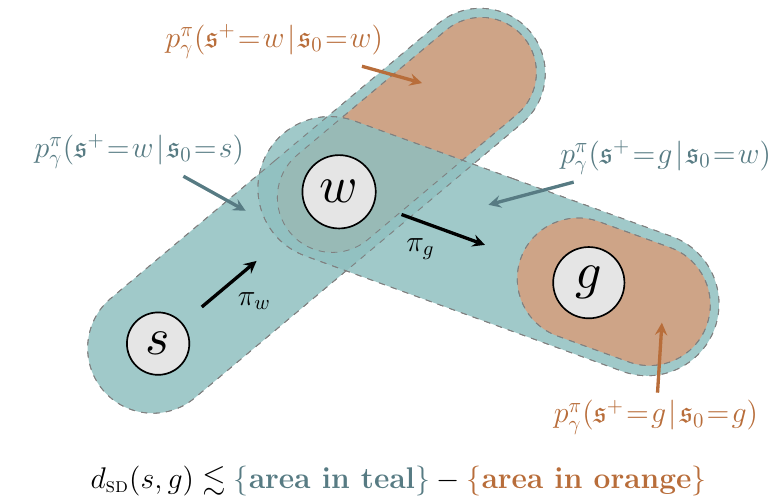}
            \subcaption{
            Our proposed distance obeys the triangle inequality. Starting at state $s$, we look at the distribution over future states ({\color{myGreen}$\blacksquare$ area}) and subtract off those states that the policy would reach starting from $w$ ({\color{myOrange}$\blacksquare$ area}). Our distance is defined as the difference in these areas, $d_\text{SD}(s, w) \triangleq {\color{myGreen}\blacksquare} - {\color{myOrange}\blacksquare}$.
            }
        \end{subfigure}
                                }
    \vspace*{-2ex}
    \caption{\makebox[0.85\linewidth][l]{An overview of our theoretical distance construction as well as the concrete implementation with metric distillation.}}
    \label{fig:overview}
\end{figure*}

\section{Introduction}
\label{sec:introduction}

Graph search is one of the most important ideas in CS, being introduced in almost every introductory CS class. However, classes often overlook a key assumption: that transitions be deterministic. With deterministic transitions, shortest-path lengths obey the triangle inequality. This property, encoded into dynamic programming algorithms, allows one to search over an exponential number of paths and find the shortest in polynomial time. This property also allows for generalization, finding new paths unseen in the data. However, in graphs (or, more generally, Markov processes) with stochastic transitions, it is unclear how to define the distance between two states such that this distance obeys the triangle inequality.

A reasonable solution for goal-reaching is to learn \emph{temporal distances}, which reflect some notion of transit time between states~\citep{venkattaramanujam2019self, savinov2018semi, durugkar2021adversarial, ma2023vip, hartikainen2019dynamical}. However, simply defining distances as hitting times breaks down in stochastic settings, as shown in prior work~\citep{akella2023distributional}.
Stochastic settings are ubiquitous in real-world problems: from autonomous vehicles navigating around drunk bar-goers, to healthcare systems rife with unobservable features.
Indeed, many advances in ML over the last decade have been predicated on probabilistic models (e.g.,  diffusion models, VAEs), so it seems rather anachronistic that an important control primitive (the notion of distances) is not well defined in a probabilistic sense.

The key challenge is that the prior notions of temporal distance break down in stochastic settings.
Nonetheless, the triangle inequality holds great appeal as a strong inductive bias for learning temporal distances: the distance between two states should be less than the length of a path that goes through a particular waypoint state.
Indeed, prior work has aimed to exploit this notion by learning ``temporal distance metrics'' that can broadly generalize from less data.

The starting point for our work is to think about distances probabilistically.
Because the dynamics may be stochastic, the number of steps it takes to traverse between two states is not a definite quantity, but rather a random variable.
To estimate the (long-term) probabilities of transiting between two states, we will build on prior temporal contrastive learning~\citep{van2019representation,eysenbach2022contrastive}, a popular and stable class of time series representation learning methods.
Intuitively, these methods learn representations from time series data so that observations that occur nearby in time are given similar representations.
Importantly, contrastive methods based on NCE and infoNCE have a probabilistic interpretation, making them ripe for application to stochastic environments.
Like prior work~\cite{eysenbach2022contrastive,eysenbach2023contrastive}, we account for the arrow of time~\cite{popper1956arrow} by using asymmetric representations, allowing the learned representations reflect the fact that (say) climbing up a mountain is more difficult from sliding back down.
The representations learned by these temporal contrastive learning methods do not themselves satisfy the triangle inequality.
However, we prove that a simple change of variables results in representations that do satisfy the triangle inequality.
Intuitively, this change of variables corresponds to subtracting off the ``distance'' between a state and itself.
Note that because the representations are asymmetric (see above), this extra ``distance'' is not zero, but rather corresponds to the likelihood of returning to the current state at some point in the future.

The main contribution of this paper is to propose a notion of temporal distance that provably satisfies the triangle inequality, even in stochastic settings.
Our constructed temporal distance is easy to learn -- simply take the features from (temporal) contrastive learning and perform a change of variables -- no additional training required!
After introducing and analyzing our proposed temporal distance, we demonstrate an application of our temporal distance to goal-conditioned reinforcement learning, using the distance function as a value function.
We use a carefully controlled synthetic benchmark to test properties such as combinatorial generalization, temporal generalization, and finding shortest paths; our results here show that the proposed distance has appealing properties that prior methods lack.
We also show that the RL method based on our distances can scale to 111-dimensional locomotion tasks, where it is competitive with prior methods on a parameter-adjusted basis.

\section{Related Work}
\label{sec:background}

Our work builds on prior work in learning temporal distances and contrastive representation learning.

\subsection{Learning distances}

Within any Markov decision process (MDP), there is an intuitive notion of ``distance'' between states as the difficulty of transitioning between them.
There are many seemingly reasonable definitions for distance a priori: likelihood of reaching the goal at a particular time, expected time to reach the goal, likelihood of ever reaching the goal, etc.
(under some policy).
The key mathematical structure for a distance to be useful for reaching goals is that it must satisfy the triangle inequality $d(a,c)\leq d(a,b)+d(b,c)$: being able to go from $a\to b$ and from $b\to c$ means going from $a\to c$ can be no harder than both of the aforementioned steps.
Such a distance is called a \textit{metric} over the state space if it is symmetric and more generally a \textit{quasimetric} \cite{paluszy1931quasi}.

While prior work on bisimulation~\citep{hansen2022bisimulation, ferns2011bisimulation} use a reward function to construct such a distance, our aim will be to define a notion of distance that does not require a reward function.

For the correct choice of distance, learning a goal-conditioned value function will correspond to selecting a distance metric that best enables goal reaching.
Such a distance can then be learned with an architecture that directly enforces metric properties, e.g.,
Euclidean distance, metric residual network (MRN), interval quasimetric estimator (IQE), etc.
\cite{wang2022improved,wang2022learning,liu2023metric}.
Since the space of value (quasi)metrics imposes a strong induction bias over value functions, using the right metric architecture can enable better combinatorial and temporal generalization \textit{without} requiring additional samples \cite{wang2023optimal}.

In deterministic MDPs, these notions of distance all coincide with distance $d(s,g)$ being proportional to the (minimum) amount of time needed to reach the goal $g$ when starting in state $s$.
Approaches like Quasimetric RL \cite{wang2023optimal,liu2023metric} learn this notion of distance, allowing optimal goal reaching in deterministic MDPs.
In general MDPs, alternative notions of distance are required \cite{akella2023distributional,n2013decision,ma2023vip,n2013decision,lelanMetricsContinuityReinforcement2021,hejna2023distance}.
Existing approaches are often limited by assumptions such as symmetry or fail to satisfy metric properties.
Our contribution is to construct a general formulation for a quasimetric over MDPs that can be easily learned from discounted state occupancy measures.

\subsection{Contrastive Representations}

Contrastive learning has seen widespread adoption for learning to represent time series \cite{van2019representation,mikolov2013distributed}.
These representations can be trained to approximate mutual information without requiring labels or reconstruction \cite{gutmann2010noise,mazoure2020deep,wu2021rethinking,van2019representation,gutmann2012noise}, and are useful for learning self-supervised representations across broad application areas \cite{radford2021learning,sermanet2017time,qian2021spatiotemporal,chen2020international,chen2021intriguing,saunshi2019theoretical,wang2020understanding, saunshi2019theoretical}.

Within RL, contrastive learning can be used for goal-conditioned control  as successor features \cite{barreto2017neural,eysenbach2022contrastive,eysenbach2023contrastive}.
Approaches that use contrastive representations for control are typically limited in combinatorial and temporal generalization since they do not bootstrap value functions \cite{zheng2023contrastive}.
Unlike past approaches that use contrastive learning for decision-making, we show that these generalization capabilities \emph{can} be obtained from contrastive successor features by imposing an additional metric structure.

\subsection{Goal-conditioned reinforcement learning (GCRL)}

Goal-reaching presents an attractive formulation for learning useful behaviors in unsupervised RL settings \cite{lairdSOARArchitectureGeneral1987,kaelblingLearningAchieveGoals}.
Recent advances in deep reinforcement learning have renewed interest in this problem as many real-world offline and online RL problems lack clear reward signals \cite{andrychowicz2017neural,eysenbach2022contrastive,park2023hiql,yang2023what,ghosh2019learning}.
GCRL methods can learn goal-conditioned policies \cite{yang2022rethinking,ghosh2021learning}, value functions \cite{eysenbach2021international,ghosh2023reinforcement}, and/or representations that enable goal-reaching \cite{eysenbach2022contrastive,zheng2023contrastive,ma2023vip}.
Approaches that recover goal-conditioned policies can also enable additional capabilities like planning \cite{fang2023generalization,chane2021goal}, skill discovery \cite{mendonca2021discovering,park2023metra} and interface with other forms of task specification like language \cite{ma2023liv,myers2023goal,shah2023mutex,black2023zero,touati2021learning}.

These GCRL techniques typically require bootstrapping with a learned value function, which can be costly and unstable, or struggle with long-horizon combinatorial and temporal generalization \cite{ghugare2024closing}.
Our approach avoids both of these shortcomings by learning a distance metric that can implicitly combine behaviors without bootstrapping or making any assumptions about the environment dynamics.

\section{General distances for goal-reaching}
\label{sec:distances}

In this section, we introduce a novel distance metric for goal-reaching in controlled Markov processes.
We show that this distance is a quasimetric, i.e., a metric that relaxes the assumption of symmetry.
In the subsequent section (\ref{sec:contrastive}), we show that this distance construction can enable additional generalization capabilities through a choice of model parameterization for temporal contrastive learning.

\subsection{Preliminaries}
We consider a discrete \textit{controlled Markov process} $M$ consisting of states $s\in\S$, actions $a\in \A$, dynamics $\dyn(s'\mid s,a)$, initial state distribution $p_0(\so=s_0)$, and a discount $\y\in(0,1)$.

By augmenting $M$ with the reward for any fixed goal $g\in\S$, which we define as $$r_g(s)=(1-\y)\,\delta_{g}(s),$$ where $\delta_{g}(s) = \Bigl\{\smqty{
        1 & \text{if } s=g\hfill   \\
        0 & \text{otherwise}\hfill
    }$ is the Kronecker delta,  we can extend $M$ to a goal-dependent Markov decision process $M_g$.
Denote by $\spol$ the (compact) set of stationary of policies $\pi(a\mid s)$ on $M$.
We also define $\nspol\supset\spol$ to be the set of non-Markovian policies $\pi(a_t\mid s_0\ldots s_t)$.
We can then derive the optimal goal-conditioned value function,\begin{equation}
    V_g^*(s) = \max_{\pi\in\Pi} \ppiy(\sp\smeq g \mid \srv_0\smeq s),
    \label{eq:goal_value}
\end{equation}
where the \textit{discounted state occupancy measure} $\ppiy$ is defined as the discounted distribution over future states $\srv^{+}$, \begin{align*}
     & \ppiys{s'}{s}
    = {(1-\gamma)}\, \sum_{k=0}^\infty \y^k \ppi(\srv_{k}\smeq s'\smid \so\smeq s)                                                                   \\
     & \text{where}\,\ \ppi(\srv_{t+1}\smeq s'\smid \srv_t=s) \smeq \sum_{a\in\A} \pi(a\smid s)\dyn(s'\smid s,a). \eqmark\label{eq:succesor_density}
\end{align*}
i.e., the distribution of $\srv^{+} \triangleq \srv_K$ for $K\sim \operatorname{Geom}(1-\gamma)$.

Here, $\srv_{t}$ denotes the state at time $t$ as a random variable, and $\sp$ denotes the state at a geometrically distributed time in the future. When needed, under a policy $\pi$, we will additionally use the notation $\arv_t$ and  $\arv^+$ to denote actions as random variables, defined analogously to $\srv_t$ and $\srv^+$.

Since \eqref{eq:goal_value} is the optimal value function corresponding to the reward $r_g$, there will always be a stationary optimal goal-reaching policy $\pi^g\in\spol$ that attains the max in \eqref{eq:goal_value}.

We can additionally view the setting of an \textit{uncontrolled Markov process} (i.e., a Markov chain) as a special case of controlled Markov processes where there is a single action $\A=\{a\}$ with a fixed policy $\Pi=\{\pi\}$.

To reason about the effects of actions, we can also consider the natural generalization of the \textit{successor state-action distribution}, which is the distribution over future states and actions $s',a'$ given that action $a$ is taken in state $s$ under $\pi$:
\begin{align*}
     & \ppiysa{g,a'}{s,a} =                                                                                     \\
     & \qquad (1-\gamma)\delta_{s,a}(g,a') + (1-\gamma)\gamma\Bigl[ \sum_{k=0}^{\infty} \sum_{s'\in\S}          \\
     & \qquad \gamma^{k} \ppi \bigl( \srv_{k}=g \mid \srv_0=s' \bigr) \pi(a' \mid g) \dyn(s' \mid s, a) \Bigr].
    \eqmark \label{eq:sa_discounted_def}
\end{align*}

Finally, we recall the definition of a quasimetric space:
\begin{definition}
    A quasimetric on $S$ is a function $d:S\times S\to \mathbb{R}$ satisfying the following for any $x, y, z \in S$.
    \begin{parts}
        \item[Positivity:] $d(x,y)\geq 0$
        \item[Identity:] $d(x,y)=0 \iff x = y$
        \item[Triangle inequality:] $d(x,z)\leq d(x,y) + d(y,z)$
    \end{parts}
    \label{def:quasimetric}
\end{definition}

\subsection{Our Proposed Temporal Distance}

With these definitions in place, we can now define the proposed temporal distance. We will start by describing a ``strawman'' approach, and then proceed with the full method.

Motivated by prior work on successor representations~\citep{dayan1993improving} and self-predictive representations~\citep{schwarzer2020data,ni2024bridging}, a candidate temporal distance is to directly use the critic function from temporal contrastive learning. When positive examples are sampled from the discounted state occupancy measure, this critic has the following form:
\begin{equation}
    -d(s, g) = \log \left( \tfrac{\ppiys{g}{s}}{p(g)} \right). \tag{not a quasimetric}
\end{equation}

However, a distance defined in this way does not satisfy the identity property of the quasimetric; namely, the distance between a state and itself can be non-zero.
Our solution is to subtract off the ``extra distance" between a state and itself, $\tilde{d}(s, g) = d(s, g) - d(g, g)$.
Doing this results in the proposed temporal distance that we propose in this paper.

We now proceed with our main definition, which is a temporal distance that obeys the triangle inequality (and is a quasimetric) even in stochastic settings.
We provide two definitions, one for controlled Markov processes and one for (uncontrolled) Markov processes:

\begin{tcolorbox} [colback=white!95!black]
    \begin{definition}
        We define the \textbf{successor distance} for a controlled Markov processes by:
        \begin{equation}
            \dss(s, g) \triangleq \min_{\pi\in\pols} \log \paren[\Bigg]{\frac{\ppiys{g}{g}}{\ppiys{g}{s}}}, \label{eq:basic_quasimetric_definition}
        \end{equation}
                As a special case for an uncontrolled Markov process, we can define:
        \begin{equation}
            \dss(s,g) \triangleq \log \paren[\Bigg]{\frac{\ppiys[]{g}{g}}{\ppiys[]{g}{s}}}.
            \label{eq:basic_quasimetric_definition_uncontrolled}
            \end{equation}
    \end{definition}
\end{tcolorbox}

To use these distances for control in model-free settings, we can extend this notion to include actions, yielding a distance over $\S\times \A$.
\begin{definition}
    We define the \textbf{successor distance with actions} for a controlled Markov process by:
    \begin{align*}
        \dss & \bigl((s, a), (g, a')\bigr) \triangleq                                                       \\
             & \quad \min_{\pi\in\pols} \biggl(\log \frac{\ppiysa{g,a'}{g,a'}}{\ppiysa{g,a'}{s,a}}\biggr).
        \eqmark
        \label{eq:basic_quasimetric_definition_actions}
    \end{align*}
\end{definition}
We make two brief lemmas about this definition; the proofs can be found in \cref{app:dist_remarks}.
Within $\S \times  \A$,
we can also say:
\begin{restatable}{lemma}{sa_indep}
    \label{thm:sa_indep}
    $\dss\bigl( (s,a), (s',a') \bigr) $  is independent of $a'$ when $s \neq s'$.
\end{restatable}
In light of this independence, we denote $\dss\bigl( s,a,s' \bigr) \triangleq \dss\bigl( (s,a), (s',a') \bigr)$ where applicable. Selecting actions that minimizes this distance corresponds to policy improvement:
\begin{restatable}{lemma}{state_min_optimal}
    \label{thm:state_min_optimal}
    Selecting actions to minimize the successor distance is equivalent to selecting actions to maximize the (scaled and shifted) Q-function:
    \begin{align*}
         & -\dss(s, a, g) = \tfrac{1}{p_g(g)}Q(s, a, g) + c_\psi(g)   \\
         & \implies \arg\max_a \dss(s, a, g) = \arg\max_a Q(s, a, g).
    \end{align*}
\end{restatable}

\paragraph{Geometric interpretation.}
Before proceeding to prove that this distance construction obeys the triangle inequality and the other quasimetric properties (\cref{sec:theory}), we provide intuition for this distance.
We visualize this distance construction in \cref{fig:overview} \textbf{(b)}.
The distribution over states visited starting at $s$ ($\ppiys{w}{s}$) is shown as the {\color{myGreen}\gscolorname\ region}; while states visited starting at $w$ ($\ppiys{w}{w}$) is shown as the {\color{myOrange}orange region}.
Our proposed distance metric is the difference in the areas of these regions (${\color{myGreen}\blacksquare} - {\color{myOrange}\blacksquare}$).
The theoretical results in the next section prove that this difference is always non-negative.
Zooming out to look at the $s$, $w$, and $g$ together, we see that these set differences obey the triangle inequality -- the area between $s$ and $g$ is smaller than the areas between $s$ and $w$ and between $w$ and $g$.
Concrete examples to build intuition for these definitions and results are presented in \cref{app:examples}.

\paragraph{Hitting times as a special case.}
To provide additional intuition into our construction, we consider a special case; the subsequent section shows that the proposed distance is a valid quasimetric in much broader settings.
In this special case, consider a controlled Markov process where the agent can remain at a state indefinitely.
This assumption means that the $\ppiys{g}{g} = 1$, so the proposed distance metric can be simplified to $d(s, g) = -\log \ppiys{g}{s}$.
This assumption also means that the hitting time of $g$ from $s$ has a deterministic value, which we will call $H(s, g)$.
Thus, we can write the discounted state occupancy measure as $\ppiys{g}{s} = \gamma^{H(s, g)}$, so the proposed distance metric is equivalent to the hitting time: $d(s, g) = H(s, g)$.
Importantly, and unlike prior work, our proposed distance continues to be a quasimetric outside of this special case, as we prove in the following section.

\subsection{Theoretical results}
\label{sec:theory}
Before proving this distance is a quasimetric over $\S$, we provide a helper lemma relating the difficulty of reaching a goal through a waypoint to the difficulty without the waypoint. The key insight we use here is that the notion of a hitting time can be generalized to represent distances in terms of discounted state occupancies.

\begin{restatable}{lemma}{waypoint}
    \label{thm:waypoint}
    For any $s,w,g\in\S$,  $\pi\in\spol$, \begin{align*}
        \max_{\pi'\in \pols} & \biggl[\frac{\ppiys[\pi']{g}{w}\ppiys[\pi]{w}{s}}{\ppiys[\pi]{w}{w}}\biggr] \\
                             & \qquad\qquad\leq \max_{\pi' \in \pols}\ppiyps{g}{s}.
    \end{align*}
\end{restatable}
The proof is in \cref{app:waypoint}. This lemma is the key to proving our main result:

\begin{tcolorbox} [colback=white!95!black]
    \begin{restatable}{theorem}{quasimetric}
        $\dss$ is a quasimetric over $\S$, satisfying the triangle inequality and other properties from \cref{def:quasimetric}.
        \label{thm:quasimetric}
    \end{restatable}
\end{tcolorbox}

The proof is in \cref{app:quasimetric}.
Compared with prior work~\citep{wang2023optimal}, our result extends to stochastic settings; we will empirically compare to this and other prior methods in \cref{sec:experiments}.

To make this result applicable to settings with unknown dynamics or without actions, we note the following corollaries:
\begin{restatable}{corollary}{quasimetric_corollary}
    \label{thm:quasimetric_corollary}
    $\dss$ is a quasimetric over $\S\times\A$.
\end{restatable}
\begin{restatable}{corollary}{quasimetric_uncontrolled_corollary}
    \label{thm:quasimetric_uncontrolled_corollary}
    $\dss$ is a quasimetric over an uncontrolled Markov process as in \cref{eq:basic_quasimetric_definition_uncontrolled}.
\end{restatable}
See \cref{app:quasimetric_corollary} for discussion of these results.

\section{Using our Temporal Distance for RL}
\label{sec:contrastive}

In this section we describe an application of our proposed temporal distance to goal-conditioned reinforcement learning.
The main challenge in doing this will be (1) estimating the successor distance defined in \cref{eq:basic_quasimetric_definition}, and (2) doing so with an architecture that respects the quasimetric properties. Once learned, we will use the successor distance as a value function for training a policy.

To introduce our methods, \cref{sec:crl} will first discuss how contrastive learning \emph{almost} estimates the successor distance. We will then introduce two variants of our method, \algname~(\alg).
The first method (CMD 1-step, \cref{sec:one_step_distillation}) will acquire the successor distance by applying contrastive learning with an energy function that is the difference of two other functions.
The second method (CMD 2-step, \cref{sec:two_step_distillation})
will acquire the successor distance by taking the features from contrastive learning and distilling those features into a quasimetric architecture.
In both cases, we then use the learned successor distance to train a goal-conditioned policy.

We emphasize that the key contribution here is the mathematical construct of \emph{what} constitutes a temporal distance, not that we use a certain architecture to represent this temporal distance. Practically, we will use the Metric Residual Network (MRN) architecture~\citep{liu2023metric} in our implementation.
Pseudocode for the full algorithms (both one-step and two-step) is provided in \Cref{alg:cmd1,alg:cmd2}. We highlight the differences between the two methods in orange for clarity.

\begin{algorithm}
    \footnotesize
    \captionsetup{font=footnotesize}
    \caption{1-step Contrastive Metric Distillation (CMD-1)}
    \label{alg:cmd1}
    \begin{algorithmic}[1]
        \STATE \algorithmicinput: batch size $B$, number of iterations $T$
        \STATE {\color{ggtext} initialize potential $\psi$, quasimetric $\phi$, and policy $\mu$ parameters}
        \STATE {\color{ggtext} define $f_{ \theta}(s, a, g) \triangleq  c_{\psi}(g) -\dmd(s, a, g)$}
        \FOR{$t=1 \ldots T$}
        \STATE sample $\{ (s_i, a_i) \sim p_s\}_{i=1}^{B}$
        \STATE sample $\{(g_{i}, a_{i}') \sim \ppiys{g_{i}}{s_i, a_i}\}_{i=1}^{B}$
        \STATE {\color{ggtext}$\phi,\psi \gets (\phi,\psi) - \alpha \nabla_{\phi,\psi} \bigl[\tilde{\L}_{\phi,\psi}^{\mathsf{c}}\left( \{s_{i},a_{i}\} ,\{g_{i}\}  \right)\bigr] $ }\hfill (\ref{eq:contrastive_loss})
        \STATE $\mu \gets \mu - \alpha \nabla_{\mu} \bigl[\L^{\pi}_{\mu}(\{s_{i},a_{i}\},\{g_{i},a_{i}'\})\bigr]$ \hfill (\ref{eq:policy_extraction})
        \ENDFOR
        \STATE \algorithmicoutput\ $\pi_{\mu}$
    \end{algorithmic}
\end{algorithm}

\begin{algorithm}
    \footnotesize
    \captionsetup{font=footnotesize}
    \caption{2-step Contrastive Metric Distillation (CMD-2)}
    \label{alg:cmd2}
    \begin{algorithmic}[1]
        \STATE \algorithmicinput: batch size $B$, number of iterations $T$
        \STATE {\color{ggtext} initialize representations $\phi, \psi$, and policy parameters $\mu$}
        \STATE {\color{ggtext} initialize quasimetric $\hat{\theta}$, margin $\lambda$}
        \STATE {\color{ggtext} define $f_{ \theta}(s, a, g) \triangleq \phi(s, a)^{T}\psi(g)$}
        \FOR{$t=1 \ldots T$}
        \STATE sample $\{ (s_i, a_i) \sim p_s\}_{i=1}^{B}$
        \STATE sample $\{(g_{i}, a_{i}') \sim \ppiys{g_{i}}{s_i, a_i}\}_{i=1}^{B}$
        \STATE $\phi,\psi \gets (\phi,\psi) - \alpha \nabla_{\phi,\psi} \bigl[\L_{\phi,\psi}^{\mathsf{c}}\left( \{s_{i},a_{i}\} ,\{g_{i},a_{i}'\}  \right)\bigr] $
        \hfill (\ref{eq:contrastive_loss_reparam})
        \STATE $\mu \gets \mu - \alpha \nabla_{\mu} \bigl[\L^{\pi}_{\mu}(\{s_{i},a_{i}\},\{g_{i},a_{i}'\})\bigr]$
        \hfill (\ref{eq:policy_extraction})
        \STATE {\color{ggtext} $ \hat{\theta} \gets \theta - \alpha \nabla_{\hat{\theta}} \bigl[\L^d_{\hat{\theta},\phi,\psi}\bigl(\{s_i,a_i\},\{g_i,a_{i}'\}\bigr)\bigr]$}
        \hfill (\ref{eq:distill_loss})
        \STATE {\color{ggtext} $\lambda \gets \lambda + \alpha \paren[\big]{\C_{\hat{\theta}}(\{s_i,a_i\},\{g_i,a_{i}'\})-\varepsilon^2}$ }
        \hfill (\ref{eq:distill_loss})
        \ENDFOR
        \STATE \algorithmicoutput\ $\pi_{\mu}$
    \end{algorithmic}
\end{algorithm}

\subsection{Building block: contrastive learning}
\label{sec:crl}

Both of our proposed methods will use contrastive learning as a core primitive, so we start by discussing how we use contrastive learning to learn an energy function $f_{ \theta}(s, a, g)$, and the relationship between that energy function and the desired successor distance.

Following prior work~\citep{eysenbach2022contrastive}, we will apply contrastive learning to learn an energy function $f_{ \theta}(s, a, g)$ that assigns high scores to $(s, a, g)$ triplets from the same trajectory, and low scores to triplets where the goal $g$ is unlikely to be visited at some point after the state-action $(s, a)$ pair.
Let $p_{sa}(s, a)$ be a marginal distribution over state-action pairs, and let $p_g(g)=\sum_{s\in\S} p_s(s)\ppiys{g }{s}$ be the corresponding marginal distribution over future states.
Contrastive learning learns the energy function by sampling pairs of state-action $(s, a)$ and goals $g$ from the joint distribution $s_{i}, a_i, g_{i} \sim  \ppiys{g_i}{s_i, a_i}p_{sa}(s_i, a_i)$. We will use the symmetrized infoNCE loss function (without resubstitution)~\cite{van2019representation,sohn2016improved,eysenbach2023contrastive}, which provides the following objective:
\def\j{{\cj}}
\begin{align}
    \min_{\theta} \E_{\{s_{i},a_{i},g_{i}\}_{i=1}^B} \mathcal{L}_{\theta}^{\mathsf{c}}\bigl(\{s_{i},a_{i}\},\{g_{i}\}\bigr).
    \eqmark \label{eq:infonce_objective}
\end{align}
given the forward and backward classification losses:
\begin{align*}
    \mathcal{L}_{\theta}^{\mathsf{c}}                                              & = \mathcal{L}_{\theta}^{\operatorname{fwd}} + \mathcal{L}_{\theta}^{\operatorname{bwd}}\eqmark \label{eq:contrastive_loss}                       \\
    \mathcal{L}_{\theta}^{\operatorname{fwd}}\bigl(\{s_{i},a_{i}\},\{g_{i}\}\bigr) & = \sum_{{\ci=1}}^{B} \log \biggl(\frac{  e^{f_{\theta}(s_{\i},a_{\i},g_{\i})} }{ \sum_{{\cj=1}}^B  e^{f_{\theta}(s_{\i},a_{\i},g_{\j})}}\biggr)  \\
    \mathcal{L}_{\theta}^{\operatorname{bwd}}\bigl(\{s_{i},a_{i}\},\{g_{i}\}\bigr) & = \sum_{{\ci=1}}^{B} \log \biggl(\frac{  e^{f_{\theta}(s_{\i},a_{\i},g_{\i})} }{ \sum_{{\cj=1}}^B  e^{f_{\theta}(s_{\j},a_{\j},g_{\i})}}\biggr).
\end{align*}
We highlight the indices $\i$ and  $\j$ for clarity.
As the batch size $B$ becomes large, the optimal critic parameters $\theta^{*}$ then satisfy~\citep{ma2018noise, poole2019variational}
\begin{align}
    f_{\theta^*}(s,a,g) = \log \biggl(\frac{\ppiys{g}{s, a}}{C \cdot p_g(g)}\biggr), \label{eq:contrastive_opt}
\end{align}
where $C$ is a free parameter. Finally, note that we can represent the successor distance \eqref{eq:basic_quasimetric_definition} as the \emph{difference} of this optimal critic evaluated on two different inputs:
\begin{align}
    f_{ \theta^*}(g, a, g) - f_{ \theta^*}(s, a, g)
     & = \frac{\ppiys{g}{g, a}}{\ppiys{g}{s, a}}. \label{eq:successor_diff}
\end{align}
The next two section present practical methods for representing this difference, either via (1) a special parameterization of this critic (\cref{sec:one_step_distillation}) or (2) distillation (\cref{sec:two_step_distillation}).

\subsection{One-step distillation (CMD 1-step):}
\label{sec:one_step_distillation}

In this section, we describe how to \emph{directly} learn the successor distance using an architecture that is guaranteed to satisfy the triangle inequality and other quasimetric properties.

The key idea is to apply the contrastive learning discussed in the prior section to a particular parameterization of the energy function, so that the difference in \cref{eq:successor_diff} is represented as a single quasimetric network.
We start by noting that the function learned by contrastive learning (Eq.~\ref{eq:contrastive_opt}) can be decomposed into the successor distance plus an additional function that depends only on the future state $g$:
\begin{align*}
     & f_{ \theta^*}(s, a, g) = \log \biggl(\frac{\ppiys{g}{s, a}}{C \cdot p_g(g)}\biggr) \eqmark \label{eq:successor_decomposition}                                                 \\
     & = \underbrace{\log \biggl(\frac{\ppiys{g}{s, a}}{\ppiys{g}{g}}\biggr)}_{-\dmd(s, a, g)} - \underbrace{\log \biggl( \frac{\ppiys{g}{g}}{C \cdot p_g(g)} \biggr)}_{-c_\psi(g)}.
\end{align*}
Thus, we will apply the contrastive objective from \cref{eq:contrastive_loss} to an energy function $f_{ \theta=(\phi, \psi)}(s, a, g)$ parameterized as the difference of a quasimetric network $d_\phi(s, a, g)$ and another learned function $c_{\psi}: \R\to \R$:
\begin{equation}
    f_{\phi,\psi}\bigl(s,a,g\bigr) = c_{\psi}(g) -\dmd(s, a, g).
    \label{eq:critic_reparam}
\end{equation}

We denote the quasimetric network as $d_\phi(s, a, g)$, with no action $a'$ at $g$ in light of \cref{thm:sa_indep}.
Since quasimetric architectures like MRN~\citep{liu2023metric} are defined using a pair of elements from the same space, we will actually parameterize the distance $d_\phi\bigl((s, a), (g, a')\bigr)$, and then enforce invariance to $a'$ by randomly permuting the $a'$ actions at the goal in the contrastive loss \cref{eq:contrastive_loss}. We can express this by restating \cref{eq:contrastive_loss} with these ``extra'' actions:
\begin{align}
    \tilde{\gL}_{\phi,\psi}^{\mathsf{c}}
    &= \tilde{\gL}_{\phi,\psi}^{\operatorname{fwd}} +
        \tilde{\mathcal{L}}_{\phi,\psi}^{\operatorname{bwd}}\eqmark
        \label{eq:contrastive_loss_reparam}\\
    \tilde{\gL}_{\phi,\psi}^{\operatorname{fwd}}\bigl(\{s_{i},a_{i}\},\{g_{i}\}\bigr)
    &= \sum_{{\ci=1}}^{B} \log \biggl(\frac{ e^{f_{\phi,\psi}(s_{\i},a_{\i},g_{\i},a'_{\j})} }{
        \sum_{{\cj=1}}^B
        e^{f_{\phi,\psi}(s_{\i},a_{\i},g_{\j},a'_{\i})}}\biggr) \nonumber\\
    \tilde{\gL}_{\phi,\psi}^{\operatorname{bwd}}\bigl(\{s_{i},a_{i}\},\{g_{i}\}\bigr)
    &= \sum_{{\ci=1}}^{B} \log \biggl(\frac{ e^{f_{\phi,\psi}(s_{\i},a_{\i},g_{\i},a'_{\j})} }{
        \sum_{{\cj=1}}^B
        e^{f_{\phi,\psi}(s_{\j},a_{\j},g_{\i},a'_{\j})}}\biggr). \nonumber \\[2pt]
      &\mspace{-130mu}\text{where }\text{ }
        f_{\phi,\psi}(s, a, g, a') \triangleq c_{\psi}(g) - d_{\phi}\bigl((s, a), (g, a')\bigr). \nonumber
\end{align}

The term $c_\psi(g)$ is important for allowing $f_{\phi,\psi}(s, a, g,a')$ to represent positive numbers, as $-\dmd\bigl((s, a), (g,a')\bigr)$ is non-positive because it is a quasimetric network.
Since the $a'$ actions in \cref{eq:contrastive_loss_reparam} are shuffled, the optimal $\theta^{*}=(\phi^{*},\psi^{*})$ allows us to state the optimal critic as $f_{\theta^{*}}(s, a, g) = c_{\psi^{*}}(g) - \dmd(s, a, g)$.
With this parameterization, we can use \cref{eq:successor_diff} to obtain the successor distance as
\begin{align*}
     & f_{ \theta^*}(g, a, g) - f_{ \theta^*}(s, a, g)                                                                               \\
     & = \cancelto{0}{-\dmd(g, a, g)} + \cancel{c_\psi(g)} + \dmd(s, a, g) - \cancel{c_\psi(g)}. \eqmark \label{eq:successor_diff_2}
\end{align*}
After contrastive learning, we will discard $c_\psi(g)$ and use $\dmd(s, a, g)$ as our successor distance. We conclude by providing the formal statement that this approach recovers the successor distance:
\begin{restatable}{lemma}{cmdonestep}
    \label{thm:cmdonestep}
    For $s\neq g$, the unique solution to the loss function in \cref{eq:contrastive_loss} with the parameterization in \cref{eq:critic_reparam} is
    \begin{align*}
        d_{\phi^{*}}(s, a, g) & = \log \frac{\ppiys{g}{s, \arv_0\smeq a}}{\ppiys{g}{g}} \eqmark \label{eq:onestep_unique}
        .\end{align*}
\end{restatable}
See \cref{app:analysis_cmdonestep} for the proof.

One appealing aspect of this approach is that it only involves one learning step. The next section provides an alternative approach that proceeds in two steps.

\subsection{Two-step distillation (CMD 2-step)}
\label{sec:two_step_distillation}

In this section we present an alternative approach to estimating the successor distance with a quasimetric network. While the first approach (CMD 1-step) is appealing because of its simplicity, this approach may be appealing in settings where pre-trained contrastive features are already available, but users want to boost performance by capitalizing on the inductive biases of quasimetric networks.

The key idea behind our approach is that that optimal critic from contrastive learning (Eq.~\ref{eq:contrastive_opt}) can be used to estimate the successor distance by performing a change of variables:
\begin{align*}
    f_{\theta} & (g,a, g) - f_{\theta}(s, a, g)                                                    \\*
               & = \log \tfrac{\ppiys{g}{g, a}}{Cp_g(g)} - \log \tfrac{\ppiys g{s,a}}{Cp_g(g)}     \\
               & = \log \tfrac{\ppiys{g}{g, a}}{\ppiys g{s,a}}. \eqmark \label{eq:dsd_upper_bound}
\end{align*}
This final expression is the successor distance; \cref{app:action_invariance} will discuss why the action $a$ in the numerator can be ignored.
Because the successor distance obeys the triangle inequality (and the other quasimetric properties), we will distill this difference into a quasimetric network. We will call this method CMD 2-Step.

\subsubsection{Distilling to a quasimetric architecture}

The representations in \cref{eq:dsd_upper_bound} already form a quasimetric on $S\times A$, and could directly be used for action selection. However, because we know that these representations satisfy the triangle inequality, distilling them into a network that is architecturally-constrained to obey the triangle inequality serves as a very strong prior: a way of potentially combating overfitting and improving generalization. To do this, we distill the bound into a distance $\dmd$ parameterized by an MRN quasimetric
\cite{liu2023metric}.

CMD 2-Step works by applying contrastive learning (Eq.~\ref{eq:contrastive_loss}). Following prior work~\cite{eysenbach2022contrastive}, we will parameterize the energy function as the inner product between learned representations: $f_{\phi,\psi}(s,a, g) =\phi(s, a)^{T}\psi(g)$.
The critic parameters are thus $\theta = (\phi, \psi)$.
We then distill the quasimetric architecture using \cref{eq:dsd_upper_bound} as a constraint. We enforce the constraint with a Lagrange multiplier $\lambda$ to ensure that the margin $\C_{\hat{\theta}}(\{s_i,a_i\},\{g_i,a_{i}'\}) $ for \cref{eq:dsd_upper_bound} satisfies
$\C_{\hat{\theta}}(\{s_i,a_i\},\{g_i,a_{i}'\})\leq \varepsilon^2$ on pairs of states and future goals sampled from the data:
\begin{align}
     & \C_{\hat{\theta}}(\{s_i,a_i\},\{g_i,a_{i}'\}) \\
     & \kern 0.1cm \triangleq \sum_{\i,\j=1}^B  \max\bigl(0,d_{\hat{\theta}}\bigl((s_{\i}, a_{\i}), (g_{\i}, a_{\j}')\bigr) - f_{\phi,\psi}(s_{\i},a_{\i}, g_{\i})\bigr)^2 \nonumber\\
     & \text{where } f_{\phi,\psi}(s,a,g)  \triangleq  \paren[\big]{\phi(g, a)-\phi(s, a)}^T\psi(g). \eqmark
    \label{eq:distillation_constraint}
\end{align}

When distilling a distance $\dss$, subject to the constraint above, we want to be maximally conservative in determining which goals we can reach. We assume \cref{eq:dsd_upper_bound} as a prior, and use dual descent to perform a constrained minimization of the objective
\begin{align*}
     & \L^d_{\hat{\theta},\phi,\psi}\bigl(\{s_{i},a_{i}\},\{g_{i}, a_{i}'\}\bigr)  \triangleq                                               \\
     & \kern 1em   \sum_{\i,\j=1}^{B}  \max\paren[\big]{0,f_{\phi,\psi}(s_{\i},a_{\i}, g_{\j})-d_{\hat{\theta}}(s_{\i}, a_{\i}, g_{\j})}^2,
    \eqmark \label{eq:distill_loss}
\end{align*}
yielding an overall optimization \begin{align}
     & \min_{\hat{\theta}} \max_{\lambda\geq 0}  \kern -1.5em {\sum_{\kern 1.5em\{s_{i},a_{i},g_{i},a_{i}'\}_{i=1}^B}}  \Bigl[
    \L^d_{\hat{\theta},\phi,\psi}\bigl(\{s_i,a_i\},\{g_i,a_{i}'\}\bigr) \nonumber                                              \\
     & \kern 2cm+ \lambda\paren[\big]{\C_{\hat{\theta}}\bigl(\{s_i,a_i\},\{g_i,a_{i}'\}\bigr)-\varepsilon^2} \Bigr].
    \label{eq:distill_optimization}
\end{align}

\subsection{Parameterizing the quasimetric}
\label{sec:parameterizing_quasimetric}
For both methods in \cref{sec:one_step_distillation,sec:two_step_distillation}, we learn a distance $d_{\theta}: (\S\times \A)^2\to \R$ parameterized with the Metric Residual Network (MRN) architecture \cite{liu2023metric}. We apply the square root correction noted by \citet[Appendix~C.2]{wang2022improved} to ensure that the distance satisfies the triangle inequality. This parameterization can be expressed using learned representations $h_{\theta},g_{\theta} : \S\times \A \to \R^{d}$:
\begin{gather*}
    d_{\theta}(x, y) = \Delta({h_{\theta}(x) - h_{\theta}(y)}) + \|g_{\theta}(x) - g_{\theta}(y)\| \\
    \text{where} \quad  \Delta(x)  = {\max}_{i=1}^d [{\max}(0,x_{i})].
    \eqmark \label{eq:mrn}
\end{gather*}

\subsection{Policy extraction}
\label{sec:policy_extraction}

Once we extract distance $\dss$, we learn a goal-conditioned policy $\pi_\mu$ to select actions that minimize the distance successor between states and random goals~\citep{schaul2015universal}:
\begin{gather*}
    \min_{\mu} \E_{p_s(s)\ p_g(g,a') \pi_{\mu}(\hat{a} \mid s, g)} \bigl[\mathcal{L}^{\pi}_{\mu}\bigl(\{s_{i},\hat{a}_{i}\},\{g_{i},a_{i}'\}\bigr)\bigr]
    \eqmark \label{eq:policy_extraction}
\end{gather*}

To prevent the policy from sampling out-of-distribution actions for offline RL~\citep{fujimoto2021minimalist, kumar2020conservative, kumar2019stabilizing}, we adopt another goal-conditioned behavioral cloning regularization from~\citet{zheng2023contrastive} or use advantage weighted regression~\citep{nair2021awac}.

With the behavior cloning regularization, the policy extraction loss becomes:
\begin{align*}
    \mathcal{L}^{\pi}_{\mu}
    &\bigl(\{s_{i},a_{i}\},\{g_{i},a_{i}'\}\bigr) =\sum_{\i,\j=1}^{B} \E_{\hat{a}\sim \pi_{\mu}(\hat{a}
        \mid s_{\i},g_{\j})} \\
    &\kern 2em \bigl[ \dmd\bigl((s_{\i}, \hat{a}), (g_{\j}, a_{\j}' )\bigr) + \log \pi_{\mu}(a_{\i}
        \mid s_{\i}, g_{\i})\bigr] .\eqmark \label{eq:policy_loss}
\end{align*}

\begin{figure*}[htb]
    \centering
    \hfill
    \begin{subfigure}[m]{.3\linewidth}
        \centering
        \includegraphics[width=\linewidth]{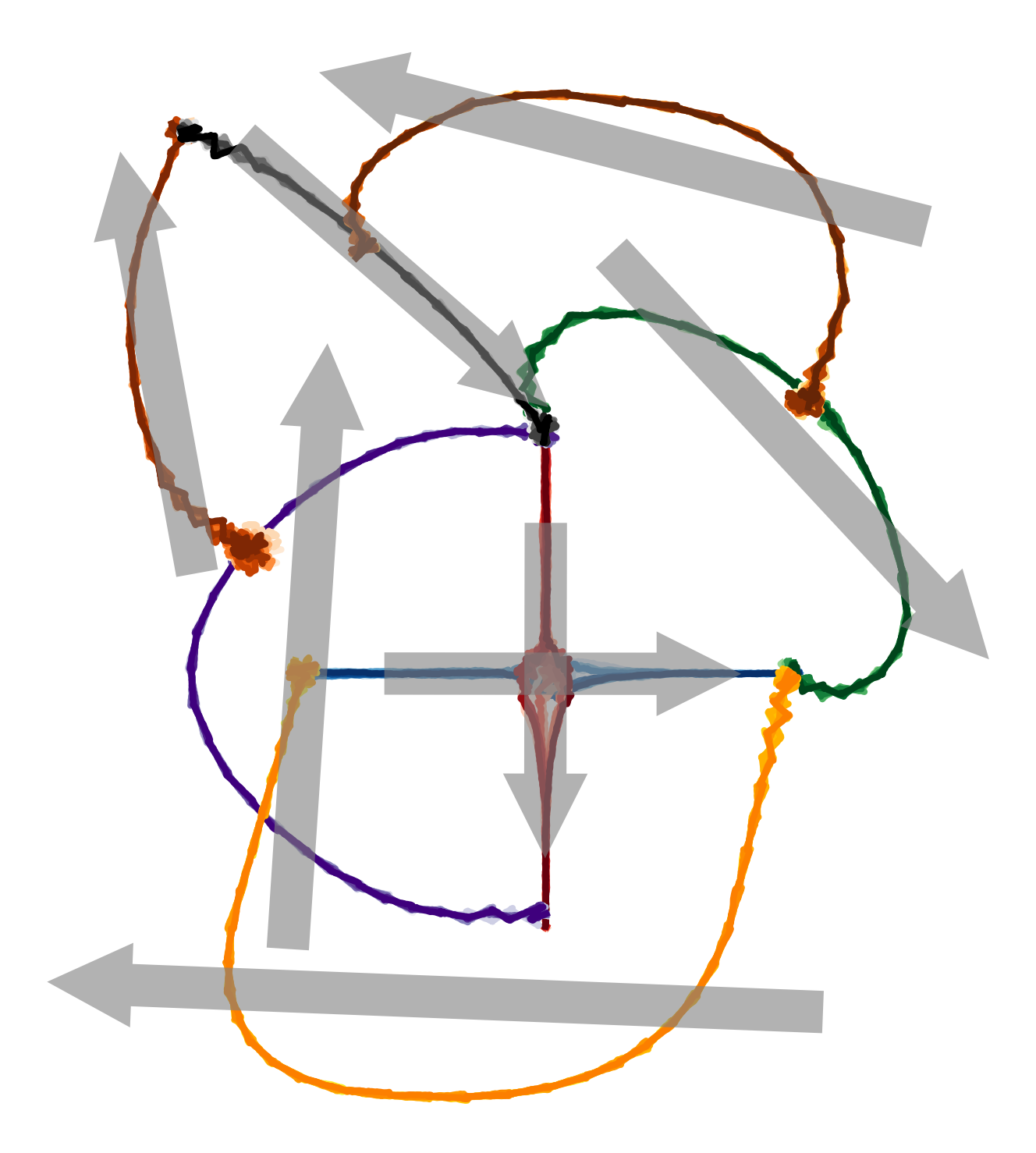}
            \end{subfigure}\hfill
    \begin{subfigure}[m]{.7\linewidth}
        \centering
        \includegraphics[width=.8\linewidth]{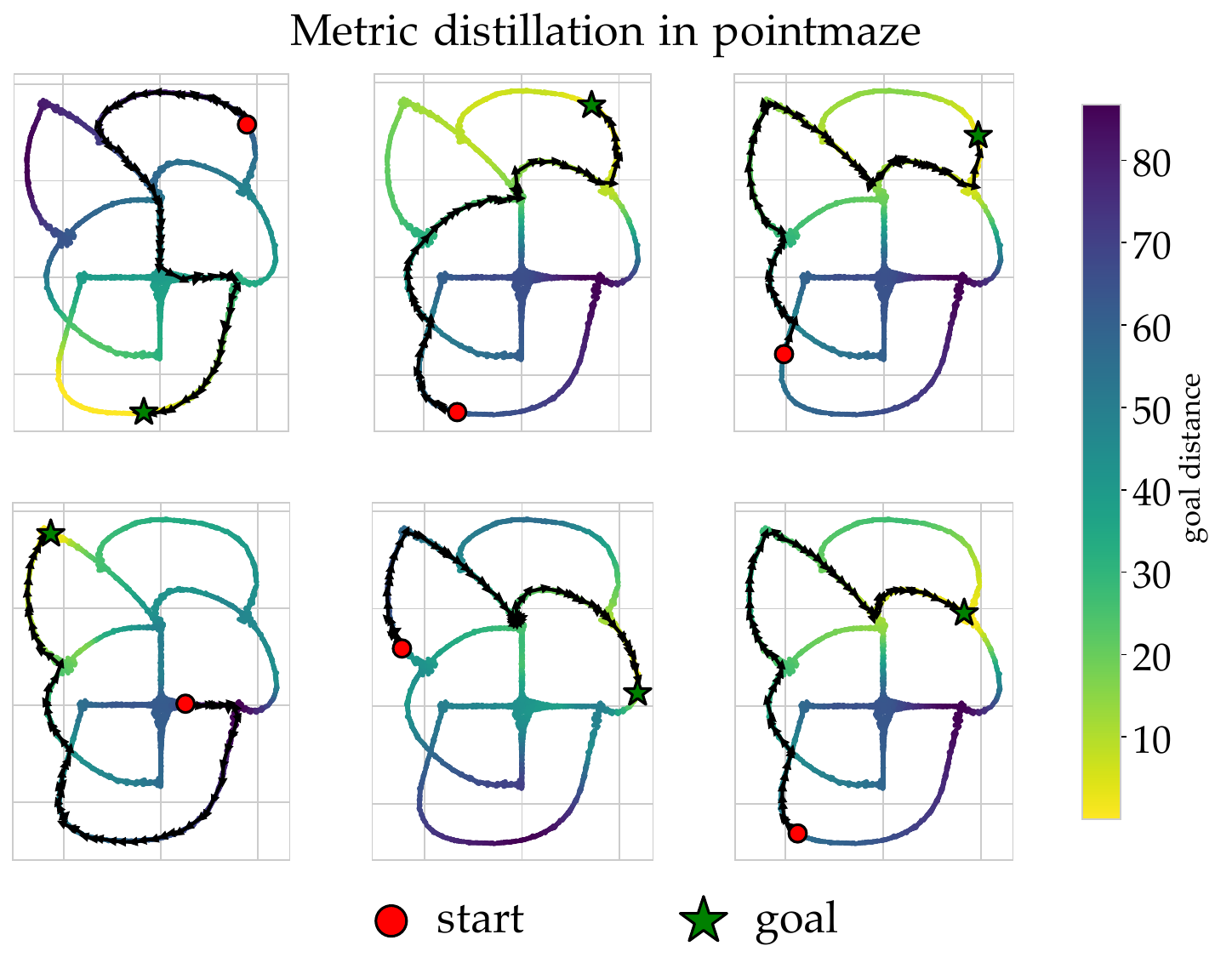}
            \end{subfigure}
    \caption{\emph{(Left)} We collect four types of trajectories on this 2D navigation task. The large gray arrows depict the direction of motion.
        Note that navigating between certain states requires piecing together trajectories of different colors.
        \emph{(Right)} Our proposed temporal distance correctly pieces together trajectories, allowing an RL agent to successfully navigate between
        pairs of states that never occur on the same trajectory. This combinatorial generalization~\citep{ghugare2024closing} or
        ``stitching''~\citep{fu2020d4rl} property is typically associated with bootstrapping with temporal difference learning, which our temporal
        distances do not require.
        \label{fig:pointmaze}}
\end{figure*}

\begin{table*}[htb]
    \centering
    \caption{\textbf{Offline RL benchmarks}: We use the AntMaze suite~\citep{fu2020d4rl} of goal-conditioned RL tasks to compare our method to prior methods, measuring the success rate and standard error across multiple seeds.
        \label{tab:offline}}
    \smallskip
    \begin{tabular}{c|cccccc}
        \toprule
                       & \textbf{CMD 1-step (Ours)} & \textbf{CMD 2-step (Ours)} & QRL                     & CRL (CPC)      & GCBC            & IQL\footnotemark \\
        \midrule
        umaze          & $90.3 \pm 4.2$             & $\mathbf{97.0 \pm 0.4}$    & $76.8 \pm 2.3$          & $79.8 \pm 1.6$ & $65.4 \pm 87.5$ & $87.5$           \\
        umaze-diverse  & $\mathbf{90.3 \pm 4.6}$    & $\mathbf{90.5 \pm 1.4}$    & $80.1 \pm 1.3$          & $77.6 \pm 2.8$ & $60.9 \pm 62.2$ & $62.2$           \\
        medium-play    & $\mathbf{78.0 \pm 4.0}$    & $72.3 \pm 2.6$             & $\mathbf{76.5 \pm 2.1}$ & $72.6 \pm 2.9$ & $58.1 \pm 71.2$ & $71.2$           \\
        medium-diverse & $\mathbf{83.0 \pm 3.1}$    & $71.8 \pm 1.0$             & $73.4 \pm 1.9$          & $71.5 \pm 1.3$ & $67.3 \pm 70.0$ & $70.0$           \\
        large-play     & $\mathbf{68.0 \pm 2.1}$    & $59.2 \pm 1.8$             & $52.9 \pm 2.8$          & $48.6 \pm 4.4$ & $32.4 \pm 39.6$ & $39.6$           \\
        large-diverse  & $\mathbf{74.5 \pm 2.3}$    & $63.6 \pm 1.9$             & $51.5 \pm 3.8$          & $54.1 \pm 5.5$ & $36.9 \pm 47.5$ & $47.5$           \\
        \bottomrule
    \end{tabular}

            \end{table*}

\section{Experiments}
\label{sec:experiments}

Our  experiments study a synthetic 2D navigation task to see whether our proposed temporal distance can learn meaningful distances of pairs
of states unseen together during training (i.e., \emph{combinatorial generalization}). We also study the efficacy of extracting policies from
this learned distance function, both in this 2D navigation setting and in a 29-dim robotic locomotion problem from the AntMaze benchmark suite. As
discussed below, for the latter experiment our comparison will be restricted to small neural network sizes. Code for our experiments is linked in \cref{app:code} and additional implementation details are provided in \cref{app:implementation}.

\subsection{Controlled experiments on synthetic data}

We first present results in a simple 2D navigation environment to illustrate how our approach can recombine pieces of data to navigate between pairs
of states unseen together during training (i.e., combinatorial generalization).

\begin{figure}
    \centering
    \includestandalone[width=.9\linewidth]{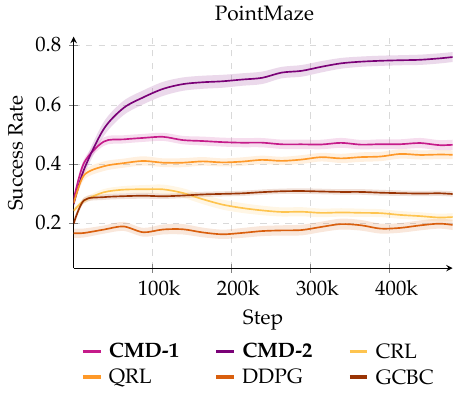}
    \caption{Metric distillation enables more efficient offline training and long-horizon compositional generalization. Results are plotted with one standard error.}
    \label{fig:training_curves}
    \vspace*{-0.5cm}

\end{figure}

We start by collecting four types of trajectories, identified in Fig.~\ref{fig:pointmaze} (left). We will be primarily interested in what distances
our method assigns to pairs of states that occur on different types of trajectories. Our hypothesis is that, by virtue of the triangle inequality, our
method will correctly reason about \emph{global distances}, despite only being trained on \emph{locally} on individual trajectories. Note that the
collected data is directed, so we will also be test whether our learned distance obeys the arrow of time.

\paragraph{Visualizing the paths.}
Using these data, we learn the contrastive representations and distill them into a quasimetric architecture, as described in \cref{sec:contrastive}.
In the subfigures in Fig.~\ref{fig:pointmaze} (right), we visualize these distances using the colormap, with the goal set to the state identified
with the ${\color{green}\bigstar}$.
This figure also visualizes paths created using the learned distances. Starting at the state identified as
${\color{red}\bullet}$, greedily select a next state within an L2 ball that has minimal temporal distance to the goal.
We repeat this process until arriving at the goal.
These planned paths demonstrate that the learned temporal distances perform combinatorial generalization; each of the subfigures
in Fig.~\ref{fig:pointmaze} show examples of inferred paths that require correctly assigning distances to pairs of states that were unseen together
during training. Note, too, that these paths follow the arrow of time: the small arrows depicting the paths go in the same direction that the data was
collected (large gray arrows in the left subplot).

\paragraph{Control performance.}

We next study whether these learned distances can be used for control, using the same synthetic dataset as above. We will compare with four baselines.
\textbf{DDPG} learns distances using Q-learning with a reward that is $-1$ at every transition until the goal is
reached~\citep{lillicrap2016continuous,lin2019reinforcement,kaelblingLearningAchieveGoals}; at least in deterministic settings, these distances should correspond to hitting
times. Quasimetric RL~\citep{wang2023optimal} is an extension of this baseline that uses a quasimetric architecture to
represent these distances.
Contrastive RL~\citep{eysenbach2022contrastive} estimates distances directly using the contrastive features (the same as used for our method), but without the metric distillation step. For all these methods as well as our method, a policy is learned using advantage-weighted maximum likelihood~\citep{neumann2008fitted, peters2007reinforcement}. We also compare with a behavioral cloning baseline, which predicts the action that was most likely to occur in the dataset conditioned on state and goal.

We measure performance by evaluating the success rate of each these approaches at reaching randomly sampled goals. In Fig.~\ref{fig:training_curves},
we plot this success rate over the course of training. Note that this experiment is done in the offline setting, so the $X$ axis corresponds to the
number of gradient steps. We observe that our temporal distance can successfully navigate to approximately 80\% of goals, while the best prior method
has a success rate of around 50\%.
Because our method starts with the same contrastive features as the contrastive RL baseline, the better performance of ours highlights the importance
of the quasimetric architecture (i.e., of imposing the triangle inequality as an inductive bias).
While both our method and quasimetric RL use a quasimetric architecture to represent a distance, we aim to represent the proposed distance
metric from \cref{sec:contrastive} while quasimetric RL aims to represent a hitting time; the better performance of our method highlights the
need to use a temporal distance that is well defined in stochastic settings such as this.

\subsection{Scaling to higher-dimensional tasks}
To study whether our temporal distance learning approach is applicable to higher-dimensional tasks, we apply it to a 111-dimensional robotic control task (AntMaze~\citep{fu2020d4rl}). In this problem setting we additionally condition the temporal distance on the action and use the learned distance as a value function for selecting actions.

We compare our approach to three competitive baselines. \textbf{GCBC} is a conditional imitation learning method that learns a goal-conditioned
policy directly, without a value function or distance function~\citep{ding2023generalizing,
    lynchLanguageConditionedImitation2021,chen2021decision, ghosh2021learning}. Both our method and Contrastive RL
(\textbf{CRL})~\citep{eysenbach2022contrastive} learn representations in the same way (Sec.~\ref{sec:contrastive}); the
difference is that our method additionally distills these representations into a quasimetric architecture. Thus, comparing our method to CRL tests
the importance of the triangle inequality as an inductive bias. We consider two variants of CRL using either rank-based
NCE~\citep{van2019representation, zheng2023contrastive} or binary-NCE~\citep{gutmann2010noise}, namely
CRL (CPC) and CRL (NCE).
Finally, Quasimetric RL (\textbf{QRL})~\citep{wang2022learning} represents a different type of temporal distance with the same quasimetric architecture as our method; it is unclear whether the temporal distance from QRL obeys the triangle inequality in stochastic settings. Thus, comparing our method to QRL tests the importance of using a temporal distance that is well defined in stochastic settings.
Prior work~\citep{zheng2023contrastive} has shown that these baselines are more competitive than other recent alternatives, including IQL~\citep{kostrikovOfflineReinforcementLearning2021} with HER~\citep{andrychowicz2017neural} and decision transformer~\citep{chen2021decision}.

\section{Conclusion}

The main contribution of this paper is a mathematical definition of temporal distance: one that obeys the triangle inequality, is meaningful in stochastic settings, and can be effectively estimated using modern deep learning techniques. Our results build upon prior work on quasimetric networks by showing how those architectures networks can be used to estimate temporal distances, including in stochastic settings. Our empirical results show that our learned distances stitch together data, allowing RL agents to navigate between states even when there does not exist a complete path between them in the training data.
Taken together, these results suggest that some elements of dynamic programming methods might be realized by simple supervised learning methods combined with appropriate architectures.

\paragraph{Limitations.} While we show that the method works effectively even on continuous settings, our theoretical results require that the MDP have discrete states. Our proposed distance may also be infinite in non-ergodic settings.
\footnotetext{IQL results are taken from \citet{kostrikovOfflineReinforcementLearning2021} which does not report standard errors.}

\FloatBarrier

\section*{Impact Statement}

This paper presents work whose goal is to advance the field of Machine Learning. There are many potential societal consequences of our work, none which we feel must be specifically highlighted here.

\section*{Acknowledgements}
We would like to thank Seohong Park, Oleg Rybkin, Micha\l{} Zawalski, Pranav Atreya, Eli Bronstein, and anonymous reviewers for discussions and feedback. We would like to acknowledge funding provided by ONR N00014-21-1-2838, AFOSR FA9550-22-1-0273, as well as NSF 2310757.
This work was also supported by Princeton Research Computing, a consortium of groups led by the Princeton Institute for Computational Science and Engineering (PICSciE) and Office of Information Technology's Research Computing.

\bibliography{references}
\bibliographystyle{icml2024}

\appendix
\onecolumn

\section{Code}
\label{app:code}

An implementation of the evaluated methods is available at \url{https://github.com/vivekmyers/contrastive_metrics}.

\section{Hitting Times}
\label{app:hitting_times}

In this section, we show several lemmas relating the discounted state occupancy measure (defined in
\cref{eq:succesor_density,eq:sa_discounted_def})
to the hitting times of states and goals. We start by defining a notion of hitting time:
\begin{definition}
    \label{def:hitting_time}
    For $\pi\in\spol$ and $s,g \in \S$, define the random variable $\hit{g}$ by \begin{align}
        \hit{g} & = \min\{t\geq 0: E_t\} \label{eq:hitting_def}                                                 \\
                & \text{ where } E_t \text{ is the event that } \srv_t = g \text{ given } \srv_0 = s. \nonumber
    \end{align}
    In other words, $\hit{g}$ is the smallest $t$ such that $\srv_t=g$ starting in $\srv_0=s$, i.e., the hitting time of $g$.
\end{definition}

Now, we can relate the discounted state occupancy measure to the hitting time of a goal.

\begin{restatable}{lemma}{hitting}
    \label{thm:hitting_time}
    For $\hit{g}$ defined as \eqref{eq:hitting_def}, \[
        \ppiys{g}{s}=\ex[][\big]{\gamma^{\hit{g}}}\ppiys{g}{g}.
    \]
\end{restatable}

\begin{proof}
    Let ${\color{myOrange} \ppi(\st = g \mid \srv_0 = s, H_s^\pi(g) = h)}$ be the probability of reaching goal $g$ at time step $t$ when starting at
    state $s$ given hitting time $H^{\pi}_s(g) = h$. By the definition of $H^{\pi}_s(g)$, we have
    \begin{align}
        {\color{myOrange} \ppi(\st = g \mid \srv_0 = s, H_s^\pi(g) = h)} & = \begin{cases}
                                                                               0                             & t < h    \\
                                                                               \ppi(\st = g \mid \srv_h = g) & t \geq h
                                                                           \end{cases}.
        \label{eq:cond_p_future_t}
    \end{align}
    Thus,
    \begin{align*}
        \ppiys{g}{s}
         & = (1-\y) \sum_{t=0}^\infty \y^t \ppi(\st = g \mid \srv_0 = s)                                                                      \\
         & = (1-\y) \sum_{t=0}^\infty \sum_{h = 0}^{\infty} \y^t \ppi(\st = g, H_s^\pi(g) = h \mid \srv_0 = s)                                \\
         & = \sum_{h = 0}^{\infty} p(H_s^\pi(g) = h) \Bigl( (1-\y) \sum_{t=0}^\infty \y^t {\color{myOrange} \ppi(\st = g \mid \srv_0 = s,
        H_s^\pi(g) = h)} \Bigr)                                                                                                               \\
         & = \sum_{h = 0}^{\infty} p(H_s^\pi(g) = h) \Bigl( (1-\y) \sum_{t = h}^\infty \y^t \ppi(\st = g \mid \srv_h = g) \Bigr) \tag{Plug in
        Eq.~\eqref{eq:cond_p_future_t}}                                                                                                       \\
         & = \sum_{h = 0}^{\infty} \y^h p(H_s^\pi(g) = h) \Bigl( (1-\y) \sum_{t = h}^\infty \y^{t - h} \ppi(\stmh = g \mid \srv_0 = g) \Bigr)
        \tag{Stationary property of MDP}                                                                                                      \\
         & = \sum_{h = 0}^{\infty} \y^h p(H_s^\pi(g) = h) \Bigl( (1-\y) \sum_{t = 0}^\infty \y^{t} \ppi(\st = g \mid \srv_0 = g) \Bigr)
        \tag{Change of variables}                                                                                                             \\
         & = \ex[][\big]{\y^{\hit{g}}}\ppiys{g}{g},
    \end{align*} as desired.
\end{proof}

We can generalize this result to account for actions as well.
\begin{definition}
    \label{def:hitting_time_actions}
    For $\pi\in\spol$, $s,g \in \S$, and $a,a'\in\A$, we define the following additional hitting time random variables \begin{align}
        H_{s,a}^{\pi}{(g,a')} & = \min\{t\geq 0: E_t\} \label{eq:hitting_def_saga}                                                      \\
                              & \text{ where } E_t \text{ is the event that } \srv_t = g, \arv_t=a' \text{ given } \srv_0 = s, \arv_0=a
        \nonumber                                                                                                                       \\
        H_{s,a}^{\pi}{(g)}    & = \min\{t\geq 0: E_t\} \label{eq:hitting_def_sag}                                                       \\
                              & \text{ where } E_t \text{ is the event that } \srv_t = g \text{ given } \srv_0 = s, \arv_0=a.
        \nonumber
    \end{align}
\end{definition}

We now show an analogous result for the discounted state-action occupancy measure.
\begin{lemma}
    \label{thm:hitting_time_actions}
    For $H_{s,a}^{\pi}{(g,a)}$ defined as \eqref{eq:hitting_def_saga} and $s\neq g$, \[
        \ppiysa{g,a'}{s,a}=\ex[][\big]{\gamma^{H_{s,a}^{\pi}{(g,a')}}} \ppiysa{g,a'}{g,a'}.
    \]
\end{lemma}

\begin{proof}
    \def\saexpr{{\color{myOrange} \ppi(\st = g, \at = a' \mid \srv_0 = s, \arv_0 = a, H_{s,a}^{\pi}{(g,a')} = h)}}Let $\saexpr$ be the probability of reaching
    goal $g$ at time step $t$ then taking action $a'$, when starting at state $s$ given the hitting time $H_{s,a}^{\pi}{(g)} = h$ and $\pi$ takes
    action $a'$ at time $h$.
    By the definition of $H_{s,a}^{\pi}{(g)}$, we have
    \begin{align}
        {\saexpr}
         & = \begin{cases}
                 0                                                      & t < h  \\
                 1                                                      & t = h  \\
                 \pi(a' \mid g)\ppi(\st = g \mid \srv_h = g, \arv_h=a') & t > h.
             \end{cases}
        \label{eq:cond_p_future_t_actions}
    \end{align}
    Thus,
    \begin{align*}
         & \hspace*{-1em}\ppiysa{g,a'}{s,a}
        \\
         & = (1-\y) \sum_{t=0}^\infty \y^t \ppi(\st = g,\at=a' \mid \srv_0 = s, \arv_0 = a)
        \\
         & = (1-\y) \sum_{t=0}^\infty \sum_{h = 0}^{\infty} \y^t \ppi(\st = g, \at=a', H_{s,a}^{\pi}{(g,a')} = h \mid \srv_0 = s, \arv_0 = a)
        \\
         & = \sum_{h = 0}^{\infty} p\bigl(H_{s,a}^{\pi}{(g,a')} = h\bigr) \Bigl( (1-\y) \sum_{t=0}^\infty \y^t {\saexpr} \Bigr)                                                                                                             \\
         & = \sum_{h =
            0}^{\infty} p\bigl(H_{s,a}^{\pi}{(g,a')} = h\bigr) (1-\y)  \Bigl( 1 + \sum_{t = h+1}^\infty \y^t \pi(a' \mid g)\ppi(\st = g \mid \srv_h = g,
            \arv_h=a' )
        \Bigr) \tag{Plug in Eq.~\eqref{eq:cond_p_future_t_actions}}                                                                                                                                                                         \\
         & = \sum_{h = 0}^{\infty} \gamma^{h} p\bigl(H_{s,a}^{\pi}{(g,a')} = h\bigr) (1-\y)  \Bigl( 1 + \sum_{t = h+1}^\infty \y^t \pi(a' \mid
        g)\ppi(\srv_{t-h} = g \mid \srv_h = g, \arv_h=a' ) \Bigr) \tag{Stationary property of MDP}                                                                                                                                          \\
         & = \sum_{h = 0}^{\infty} \gamma^{h} p\bigl(H_{s,a}^{\pi}{(g,a')} = h\bigr) (1-\y)  \Bigl( 1 + \sum_{t = 1}^\infty \y^t \pi(a' \mid
        g)\ppi(\srv_{t} = g \mid \srv_0 = g, \arv_0=a' ) \Bigr) \tag{Change of variables}                                                                                                                                                   \\
         & = \sum_{h = 0}^{\infty} \gamma^{h} p\bigl(H_{s,a}^{\pi}{(g,a')} = h\bigr) (1-\y)  \Bigl( 1 + \sum_{t = 1}^\infty \y^t \pi(a' \mid
        g)\ppi(\srv_{t} = g \mid \srv_1 = s' )\dyn(s' \mid s,a) \Bigr)                                                                                                                                                                      \\
         & = \sum_{h = 0}^{\infty} \gamma^{h} p\bigl(H_{s,a}^{\pi}{(g,a')} = h\bigr) (1-\y)  \Bigl( 1 +  \gamma\sum_{t = 0}^\infty \y^t \ppi(\srv_{t} = g \mid \srv_0 = s' )\pi(a' \mid g)\dyn(s' \mid s,a) \Bigr)\tag{Change of variables} \\
         & = \sum_{h = 0}^{\infty} \gamma^{h} p\bigl(H_{s,a}^{\pi}{(g,a')} = h\bigr) (1-\y)  \Bigl( \delta_{g,a'}(g,a') +  \gamma\sum_{t = 0}^\infty \y^t \ppi(\srv_{t} = g \mid \srv_0 = s' )\pi(a' \mid g)\dyn(s' \mid s,a) \Bigr)        \\
         & = \sum_{h = 0}^{\infty} \gamma^{h} p\bigl(H_{s,a}^{\pi}{(g,a')} = h\bigr) \ppiysa{g,a'}{g,a'}                                                                                                                                    \\
         & = \ex[][\big]{\gamma^{H_{s,a}^{\pi}{(g,a')}}} \ppiysa{g,a'}{g,a'},
    \end{align*}

\end{proof}

\begin{remark}
    \label{thm:goal_action_hitting}
    The hitting times $H_s^{\pi}(g)$ and $H_{s,a}^{\pi}(g)$ are independent of the distribution $\pi(\cdot  \mid  g)$.
\end{remark}

\begin{remark}
    \label{thm:hitting_decomposition}
    We can write \[
        H_{s,a}^{\pi}(g,a') = H_{s,a}^{\pi}(g) + \ex[\pi(\hat{a}\mid g)][\big]{H_{g,\hat{a}}^{\pi}(g,a')}.
    \]
\end{remark}

These remarks follow from the definitions in \cref{eq:hitting_def_saga,eq:hitting_def_sag} and the conditional independence of the states before $g$ is reached and the action taken at $g$.

\section{Proofs of \cref{thm:sa_indep,thm:state_min_optimal}}
 \label{app:dist_remarks}

\label{app:sa_indep}
\csname sa_indep\endcsname*

\begin{proof}
    Suppose $s\neq g$.
    We have from \cref{eq:basic_quasimetric_definition_actions} that
    \begin{align*}
        \dss & \bigl( (s,a), (g,a') \bigr)                                                                                                \\
             & = \min_{\pi\in\pols} \biggl[\log \frac{\ppiysa{g,a'}{g,a'}}{\ppiysa{g,a'}{s,a}}\biggr]                                     \\
             & = -\max_{\pi\in\pols} \Bigl[ \log \ex[][\big]{\gamma^{H_{s,a}^{\pi}{(g,a')}}} \Bigr] \tag{\cref{thm:hitting_time_actions}} \\
             & = -\max_{\pi\in\pols} \log \ex[][\big] {\gamma^{H_{s,a}^{\pi}(g)
                + \ex[\pi(\hat{a}\mid g)][]{H_{g,\hat{a}}^{\pi}(g,a')} }} \tag{\cref{thm:hitting_decomposition}}.
    \end{align*}
                                            Now, from \cref{thm:goal_action_hitting}, the first term $H_{s,a}^{\pi}(g)$ is independent of $\pi(\cdot  \mid g)$.
    Meanwhile, the second term $\ex[\pi(\hat{a}\mid g)][]{H_{g,\hat{a}}^{\pi}(g,a')}$ is minimized when $\pi( \hat{a}  \mid g) = \delta_{a'}(\hat{a})$, i.e.,
    when the action taken at $g$ is $a'$.
    Thus, at the maximum $\pi(\cdot  \mid g) = \delta_{a'}(\cdot)$; continuing, we see
    \begin{align*}
        \dss & \bigl( (s,a), (g,a') \bigr)                                            \\
             & = -\max_{\pi\in\pols} \log \ex[][\big] {\gamma^{H_{s,a}^{\pi}(g)
        + \ex[\pi(\hat{a}\mid g)][]{H_{g,\hat{a}}^{\pi}(g,a')} }}                     \\
             & = -\max_{\pi\in\pols} \log \ex[][\big] {\gamma^{H_{s,a}^{\pi}(g)  }} .
    \end{align*}

    From this last expression we see that $\dss\bigl((s,a), (g,a')\bigr)$ is independent of the action at the goal $a'$, as desired.
\end{proof}

\csname state_min_optimal\endcsname*
\label{app:state_min_optimal}

    \begin{proof}
        As noted in prior work~\citet[Lemma~4.1]{eysenbach2022contrastive}, the optimal critic (Eq.~\ref{eq:contrastive_opt}) is equivalent to a scaled Q function:
        \begin{align*}
            e^{f_{ \theta^*}(s, a, g)} = \frac{1}{C \cdot p_g(g)} \underbrace{\ppiys{g}{s, a}}_{Q(s, a, g)}.
        \end{align*}
        \cref{eq:successor_decomposition} then tells us that the successor distance differs from $f_{ \theta^*}(s, a, g)$ by a term that depends only on $g$, so taking the argmin of the successor distance is the same as taking the argmax of this scaled Q function.
    \end{proof}

\section{Proof of \cref{thm:waypoint}}
\label{app:waypoint}

Now, we will prove \cref{thm:waypoint}.
\waypoint*

\begin{proof}

    Define     $\tpi\in\nspol$ to be the non-Markovian policy that starts executing $\pi'$ and switches to $\pi$ after reaching $w$:
    \begin{align*}
        \tpi(a_t\mid s_t) & = \left\{\begin{array}{cl}
                                         \pi(a_t\mid s_t)  & w\in\curly*{s_0, s_1, \ldots, s_t} \\
                                         \pi'(a_t\mid s_t) & \text{otherwise.}
                                     \end{array}\right.
    \end{align*}
    We take $\pi' \in  \Pi$ to be an arbitrary policy.
    Let $E_1$ be the event where the hitting time of waypoint $w$ is less than the hitting time of goal $g$ starting from state $s$, i.e.,
    $E_1 = \{H_{s}^{\tpi}(w) < H_{s}^{\tpi}(g) \}$.
    Complementary, let $E_2$ be the event where the hitting time of waypoint $w$ is greater than or equal to the
    hitting time of goal $g$ starting from state $s$, i.e., $E_2 = \{H_{s}^{\tpi}(w) \ge H_{s}^{\tpi}(g) \}$. We note that $E_1$ and $E_2$ are mutually exclusive.
                        
                    \def\blueexpr{{\color{Cyan} p^{\tpi}(\st = g \mid \srv_0 = s,
                H_{s}^{\tpi}{(w)} = h)}}
    We start by rewriting $\ppiys[\tpi]{g}{s}$: \begin{align}
        \ppiys[\tpi]{g}{s} & = \sum_{h = 0}^\infty p(H_{s}^{\tpi}(w) = h) \left( (1 - \y) \sum_{t = 0}^\infty \y^t {\blueexpr} \right) \nonumber \\
                           & = \sum_{h = 0}^\infty p(H_{s}^{\pi'}(w) = h) \left( (1 - \y) \sum_{t = 0}^\infty \y^t {\blueexpr} \right).
        \label{eq:p_future_hit_expansion}
    \end{align}
                                                
    \def\pinkexpr{{\color{CarnationPink} p^{\pi}(\srv_t = g \mid \srv_h = w)}}
    \def\thit#1{H_{s}^{\tpi}{(#1)}}
    \def\phit#1{H_{s}^{\pi'}{(#1)}}
    Now, $\blueexpr$ can be written as
    \begin{align}
        \blueexpr = \begin{cases}
                        0           & t < h, \text{under } E_1 \\
                        {p^{\tpi}(\srv_t = g \mid \srv_0 = s, \thit{w} = h, \thit{g}
                        \leq h)}    & t < h, \text{under } E_2 \\
                        {\pinkexpr} & t \geq h
                    \end{cases}.
    \end{align}
            Dropping the first $h$ terms (which are all non-negative), we get
    \begin{align*}
        \sum_{t = 0}^\infty \y^t {\blueexpr} & \geq \sum_{t = h}^\infty \y^t {\blueexpr} = \sum_{t = h}^{\infty} \y^t {\pinkexpr}
    \end{align*}
    Plugging this inequality into \cref{eq:p_future_hit_expansion}, we have
    \begin{align*}
        \ppiys[\tpi]{g}{s} & \geq \sum_{h = 0}^\infty p(\phit{w} = h) \left( (1 - \y) \sum_{t = h}^\infty \y^t {\pinkexpr} \right)                                                                  \\
                           & = \sum_{h = 0}^\infty \gamma^h p(\phit{w} = h) \left( (1 - \y) \sum_{t = h}^\infty \y^{t - h} \ppi(\stmh = g \mid \srv_0 = w) \right) \tag{Stationary property of MDP} \\
                           & = \sum_{h = 0}^\infty \gamma^h p(\phit{w} = h) \left( (1 - \y) \sum_{t = 0}^\infty \y^t \ppi(\st = g \mid \srv_0 = w) \right) \tag{Change of variables}                \\
                           & = \ex[][\big]{\y^{\phit{w}}}\ppiys{g}{w}.
    \end{align*}
    Applying \Cref{thm:hitting_time} to the last step, we see
    \begin{align*}
        \ppiys[\tpi]{g}{s} & \geq \ex[][\big]{\y^{\hit{w}}}\ppiyps{g}{w}        \\
                           & = \frac{\ppiyps{g}{w} \ppiys{w}{s}}{\ppiys{w}{w}}.
    \end{align*}
    Since there is a stationary Markovian optimal policy $\pi^{*}$ for $r_g$ in $M$, we know from \cref{thm:state_min_optimal} that \[
        \ppiys[\tpi]{g}{s}\leq \max_{\pi'\in\spol} \ppiyps{g}{s},
    \] so we have \[
        \max_{\pi'\in\spol} \ppiyps{g}{s} \geq \frac{\ppiyps{g}{w} \ppiys{w}{s}}{\ppiys{w}{w}}.
    \]
    Since $\pi'$ on the RHS was arbitrary, we conclude
    \[
        \max_{\pi'\in\spol}\Bigl[\frac{\ppiyps{g}{w} \ppiys{w}{s}}{\ppiys{w}{w}}\Bigr] \leq \max_{\pi'\in\spol} \ppiyps{g}{s} .
    \]

\end{proof}

\begin{lemma}
    \label{thm:sa_waypoint}
    For any $s,w,g\in\S$, $a_s, a_w, a_g \in \A$, and $\pi \in \spol$, we have
    \begin{align*}
        \max_{\pi'\in \pols} & \biggl[\frac{\ppiysa[\pi']{g,a_g}{w,a_w}\ppiysa[\pi']{w,a_w}{s,a_s}}{\ppiysa[\pi']{w,a_w}{w,a_w}}\biggr] \\
                             & \qquad\qquad\qquad\qquad\qquad \leq \max_{\pi' \in \pols}\bigl[\ppiysa[\pi']{g,a_g}{s,a_s}\bigr].
    \end{align*}
\end{lemma}

The proof follows from the same argument as in \cref{thm:waypoint} but applying \cref{thm:hitting_time_actions} instead of \cref{thm:hitting_time}.

\section{Proof of \cref{thm:quasimetric}}
\label{app:quasimetric}

\quasimetric*

\begin{proof}
    We check the conditions of \Cref{def:quasimetric}:
    \begin{parts}

        \item[Positivity:] Applying \Cref{thm:hitting_time}, we see \begin{align*}
            \dss(s,g) & =\min_{\pi\in\spol} \log \ppiys{g}{g} - \log \ppiys{g}{s}                           \\
                      & = \min_{\pi\in\spol} \log \ppiys{g}{g} - \log \ex[][\big]{\y^{\hit{g}}}\ppiys{g}{g} \\
                      & \geq \min_{\pi\in\spol} \log \ppiys{g}{g} - \log \ppiys{g}{g}                       \\
                      & = 0.
        \end{align*}

        \item[Identity:] We see $\dss(s,g)=0$ precisely iff $\ppiys{g}{g}=\ppiys{g}{s}$ for some $\pi\in\spol$. This holds when $s=g$. For $s\neq g$, we
        have $\ppiys{g}{s} \leq \gamma  \ppiys{g}{g}$. Since $\ppiys{g}{g} \geq 1-\y$ by construction, $\dss(s,g)\neq 0$.

        \item[Triangle inequality:]
        We see: \begin{align*}
            \dss(s,g) & = \min_{\pi\in\spol} \log \ppiys{g}{g} - \log \ppiys{g}{s}
            \\
                      & \leq \min_{\pi\in\spol} \log \ppiys{g}{g} - \log\paren[\bigg]{\max_{\pi'\in\spol}
            \bracket[\bigg]{\frac{\ppiys{g}{w}\ppiys[\pi']{w}{s}}{\ppiys[\pi']{w}{w}}}} \tag{\Cref{thm:waypoint}} \\
                      & = \min_{\pi\in\spol} \log \ppiys{g}{g} -
            \max_{\pi'\in\spol}\log\paren[\bigg]{{\frac{\ppiys{g}{w}\ppiyps{w}{s}}{\ppiyps{w}{w}}}}
            \\
                      & = \paren*{\min_{\pi\in\spol} \log \frac{\ppiys{g}{g} }{
            \ppiys{g}{w}}}-\paren[\bigg]{{\max_{\pi'\in\spol}\log{\frac{\ppiyps{w}{s}}{\ppiyps{w}{w}}}}}          \\
                      & = \paren*{\min_{\pi\in\spol} \log \frac{\ppiys{g}{g} }{ \ppiys{g}{w}}}
            +\paren[\bigg]{{\min_{\pi'\in\spol}
                        \log{\frac{\ppiyps{w}{w}}{\ppiyps{w}{s}}}}}
            \\
                      & = \dss(w,g) + \dss(s,w)
            \eqmark\label{eq:waypoint_elbo}
        \end{align*}
        as desired.
    \end{parts}
\end{proof}

\label{app:quasimetric_corollary}

Consider the following didactic example for why we might want to extend the successor distance to the state-action space $\S\x\A$
(Eq.~\ref{fig:sa_metric}).

\begin{figure}[H]
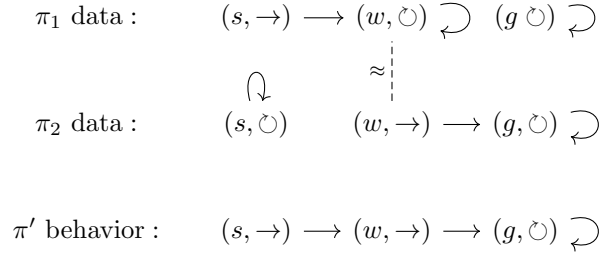

    \centering
    \includestandalone{images/sa-dist-example}
    \caption{A simple illustration of a metric over $\S\times\A$.
        To stitch the behavior $s\to w$ from $\pi_1$ and $w\to g$ from $\pi_2$ to the behavior $s\to g$ that is possible under some policy $\pi$', we
        enforce an additional constraint that distances to $(w,\to)$ are the same as distances to $(w,\circlearrowright)$.}
    \label{fig:sa_metric}
\end{figure}

\csname quasimetric_corollary\endcsname*

This statement follows from the same argument as in \cref{thm:quasimetric} but applying \cref{thm:sa_waypoint} instead of \cref{thm:waypoint} to the triangle inequality.

\csname quasimetric_uncontrolled_corollary\endcsname*

This statement follows from \cref{thm:quasimetric} by taking $\A=\{a\}$ so $\Pi = \{\pi\}$.

\section{Analysis of \alg-1}
\label{app:analysis_cmdonestep}
\cmdonestep*
\begin{proof}
    \cref{eq:successor_diff} together with \cref{eq:successor_diff_2}  tell us that, if $f(s, a, g)$ satisfies \cref{eq:contrastive_opt}, then the learned $\dmd(s, a, g)$ is the successor distance. What remains is to show that the parameterization in~\cref{eq:critic_reparam} is sufficient to represent \cref{eq:contrastive_opt}: \cref{eq:successor_decomposition} tells us that it is sufficient for (1) we use a universal quasimetric network for $\dmd(s, a, g)$~\citep{liu2023metric, wang2022improved}, and (2) use a universal  network  $c_\psi(g)$ (e.g., sufficient layers in a neural network~\citep{hornik1989multilayer}).
\end{proof}

\section{Implementation Details}
\label{app:implementation}

We implement CMD, CRL (CPC / NCE), and GCBC using JAX  building upon the official codebase of contrastive RL~\citep{eysenbach2022contrastive}.
For the QRL baseline, we use the implementation provided by the author~\citep{wang2022learning}.
Whenever possible, we used the same hyperparameters as contrastive RL~\citep{eysenbach2022contrastive} and match the number of parameters in the model for different algorithms.
We used 4 layers of 512 units of MLP as our neural network architectures and set batch size to 256.
We find that using a smaller learning rate $5\cdot 10^{-6}$ for the contrastive network is useful for improving performance.
In light of \cref{thm:sa_indep}, when learning the $\dss$ critic in \cref{eq:onestep_unique,eq:dsd_upper_bound}, we use a dummy action  $a'$ sampled from the marginal distribution over geometrically-discounted future actions.

We compared approaches in the offline settings across the best performance from 500k steps of training, consistent with past work \cite{zheng2023contrastive,eysenbach2022contrastive}.
All approaches were tested with similar model sizes and runtime, and used tuned hyperparameters. Our code at \url{https://github.com/mnm-anonymous/qmd} features the precise configurations for the experiments.

\section{Worked Examples}
\label{app:examples}

We present a few examples of how the successor distance defined in \cref{eq:basic_quasimetric_definition_uncontrolled} yields a valid quasimetric.

\subsection{Example 1: 3-state Markov Process.}

\begin{center}
    \begin{tikzcd}
        1 \rar & 2 \rar & 3 \arrow[loop,looseness=5]
    \end{tikzcd}
\end{center}

\cref{eq:succesor_density}
Assume that the initial state is ``1'', a discount factor of $\gamma$, and that state ``3'' is absorbing. We assume that the discounted state occupancy measure states at $t = 0$, so that it includes the current time step.

\begin{align*}
    p(3 \mid 3)       & = 1                                                                                                          \\
    p(2 \mid 2)       & = 1 - \gamma                                                                                                 \\
    p(2 \mid 1)       & = \gamma (1 - \gamma)                                                                                        \\
    p(3 \mid 2)       & = \gamma                                                                                                     \\
    p(3 \mid a)       & = \gamma^2                                                                                                   \\
    d(1, 3)           & = \log p(3 \mid 3) - \log p(3 \mid 1) = \log 1 - \log \gamma^2 = 0 + 2 \log \tfrac{1}{\gamma}                \\
    d(1, 2)           & = \log p(2 \mid 2) - \log p(2 \mid 1) = \log (1 - \gamma) - \log \gamma(1 - \gamma) = \log \tfrac{1}{\gamma} \\
    d(2, 3)           & = \log p(3 \mid 3) - \log p(3 \mid 2) = \log 1 - \log \gamma = \log \tfrac{1}{\gamma}                        \\
    d(1, 2) + d(2, 3) & = 2 \log \tfrac{1}{\gamma} \ge d(1, 3) = 2 \log \tfrac{1}{\gamma}. \checkmark
\end{align*}
In this example, note that the triangle inequality is tight. This is because there is a single state that we are guaranteed to visit between states ``1'' and ``3.''

\subsection{Example 2: 4-state Markov Process.}

\begin{center}
    \begin{tikzcd}
        1 \rar & \begin{matrix}
            2 \\ 3
        \end{matrix} \rar & 4 \arrow[loop,looseness=5]
    \end{tikzcd}
\end{center}

From state ``1'', states ``2'' and ``3'' each occur with probability 0.5.
\begin{align*}
    p(4 \mid 4)       & = 1                                                                                    \\
    p(2 \mid 2)       & = 1 - \gamma                                                                           \\
    p(2 \mid 1)       & = \tfrac{1}{2}(1 - \gamma)\gamma                                                       \\
    p(4 \mid 1)       & = \gamma^2                                                                             \\
    p(4 \mid 2)       & = \gamma                                                                               \\
    d(1, 2)           & = \log p(2 \mid 2) - \log p(2 \mid 1)                                                  \\&= \log (1 - \gamma) - \log \tfrac{1}{2}(1 -\gamma)\gamma = \log \tfrac{1}{\gamma} + \log 2\\
    d(2, 4)           & = \log p(4 \mid 4) - \log p(4 \mid 2)                                                  \\&= \log 1 - \log \gamma = \log \tfrac{1}{\gamma} \\
    d(1, 4)           & = \log p(4 \mid 4) - \log p(1 \mid 4)                                                  \\&= \log 1 - \log \gamma^2 = 2 \log \tfrac{1}{\gamma} \\
    d(1, 2) + d(2, 4) & = 2 \log \tfrac{1}{\gamma} + \log 2 \ge d(1, 4) = 2 \log \tfrac{1}{\gamma}. \checkmark
\end{align*}
In this example, the triangle inequality is loose. This is because we have uncertainty over which states we will visit between ``1'' and ``4.'' One
way to resolve this uncertainty is to aggregate states ``2'' and ``3'' together; if we did this, we'd be back at example 1, where the triangle
inequality is tight.

\newpage

\section{Action-Invariance}
\label{app:action_invariance}

Let's assume that data are collected with a Markovian policy, so $p(s', a' \mid s, a) = \beta(a' \mid s')p(s' \mid s, a)$. Then CRL will learn
\begin{align}
    e^{f(s, a, s', a')} & = \frac{p(s', a' \mid s, a)}{p(s', a')}                                              \\
                        & = \frac{\cancel{\beta(a' \mid s')}p(s' \mid s, a)}{\cancel{\beta(a' \mid s')}p(s')}.
\end{align}
Thus, if data are collected with a Markovian policy, then the optimal critic will not depend on the future actions.
Note that this remains true for any parameterization of the critic (including MRN) that can represent the optimal critic.

However, the assumption on a Markovian data collection policy can be violated in a few ways:
\begin{enumerate}
    \item In the online setting, data are collected from policies at different iterations. In this setting, conditioning on a previous state and action can give you a better prediction of $a'$ (violating the Markov assumption) because it can allow you to infer which policy you're using.
    \item In goal-conditioned settings, the data collection policy is conditioned on the goal. Conditioning on a previous state and action can leak information about the desired goal.
\end{enumerate}

One way of fixing this is to apply CRL to a different data distribution. Let $p(s', a' \mid s, a)$ be given, and let $\beta(a)$ be some distribution over actions (in practice, we might use the marginal distribution over actions in the dataset). Define
\begin{align}
    \tilde{p}(s', a' \mid s, a) \triangleq p(s' \mid s, a)\beta(a'), \qquad \tilde{p}(s', a') \triangleq p(s')\beta(a').
\end{align}
In practice, this corresponds to augmenting the CRL training examples $(s, a, s', a') \rightarrow (s, a, s', \tilde{a}')$ by resampling the future actions. Now, consider applying CRL to this new distribution:
\begin{align}
    e^{f(s, a, s', a')} & = \frac{\tilde{p}(s', a' \mid s, a)}{\tilde{p}(s', a')}                                              \\
                        & = \frac{\cancel{\beta(a' \mid s')}\tilde{p}(s' \mid s, a)}{\cancel{\beta(a' \mid s')}\tilde{p}(s')}.
\end{align}
Thus, if we apply CRL to data augmented in this way, we're guaranteed to learn a critic function $f(s, a, s', a')$ that is invariant to $a'$.

\end{document}